%% file: main-icaps.tex
\documentclass[letterpaper]{article} 
\usepackage{aaai25}
\usepackage{times}  
\usepackage{helvet}  
\usepackage{courier}  
\usepackage[hyphens]{url}  
\usepackage{graphicx} 
\usepackage{natbib}  
\usepackage{caption} 
\usepackage{algorithm}
\usepackage{algorithmic}

\usepackage{amsmath}
\usepackage{amssymb}
\usepackage{mathtools}
\usepackage{amsthm}

\usepackage{pgf}
\usepackage{tikz}
\usetikzlibrary{arrows.meta, positioning, fit, calc, shapes, backgrounds}
\usepackage{xcolor}
\definecolor{PastelPink}{HTML}{EECCD3}
\definecolor{Teal}{HTML}{80C4B7}
\definecolor{ButterYellow}{HTML}{EEC95C}
\definecolor{PapayaOrange}{HTML}{E3856B}
\usepackage{amsthm}
\theoremstyle{remark}

\theoremstyle{definition}
\newtheorem{definition}{Definition}

\usepackage{thmtools, thm-restate}

\usepackage{xspace}
\DeclareRobustCommand{\eg}{e.g.,\@\xspace}
\DeclareRobustCommand{\ie}{i.e.,\@\xspace}
\DeclareRobustCommand{\wrt}{w.r.t.\@\xspace}

\usepackage{newfloat}
\usepackage{listings}
\definecolor{codegreen}{rgb}{0,0.6,0}
\definecolor{codegray}{rgb}{0.5,0.5,0.5}
\definecolor{codepurple}{rgb}{0.58,0,0.82}
\definecolor{backcolour}{rgb}{0.95,0.95,0.92}
\lstdefinelanguage{PDDL}{
  sensitive = true,
  keywords = [1]{domain, define, types, requirements, predicates, action, parameters, precondition, effect},
  morekeywords=[2]{and, or, not, exists, forall}
  }
\DeclareCaptionStyle{ruled}{labelfont=normalfont,labelsep=colon,strut=off} 
\floatstyle{ruled}
\newfloat{listing}{tb}{lst}{}
\floatname{listing}{Listing}
\title{Projection Abstractions in Planning Under the Lenses of Abstractions for MDPs}
\author{
   Giuseppe Canonaco \textsuperscript{\rm 1},
   Alberto Pozanco \textsuperscript{\rm 1},
   Daniel Borrajo \textsuperscript{\rm 1}
}
\affiliations{
    \textsuperscript{\rm 1}JPMorganChase AI Research\\

}

\usepackage{bibentry}

\begin{document}
\include{macros}
\maketitle

\begin{abstract}
The concept of abstraction has been independently developed both in the context of AI Planning and discounted Markov Decision Processes (MDPs). However, the way abstractions are built and used in the context of Planning and MDPs is different even though lots of commonalities can be highlighted. To this day there is no work trying to relate and unify the two fields on the matter of abstractions unraveling all the different assumptions and their effect on the way they can be used. Therefore, in this paper we aim to do so by looking at projection abstractions in Planning through the lenses of discounted MDPs. Starting from a projection abstraction built according to Classical or Probabilistic Planning techniques, we will show how the same abstraction can be obtained under the abstraction frameworks available for discounted MDPs. Along the way, we will focus on computational as well as representational advantages and disadvantages of both worlds pointing out new research directions that are of interest for both fields.
\end{abstract}

%

\section{Introduction}
\label{sec:intro}
Abstractions are tools that help agents to simplify their reasoning when solving problems. In the context of AI Planning this translates into automatically building a new, simpler problem that can be solved fast. This solution is used to guide and can possibly accelerate the search for a solution in the original task, representing the true goal of the agent. The most common Planning abstractions are defined through aggregation functions that disregard a subset of the state-space components, \ie projections~\cite{edelkamp2002symbolic,helmert2007flexible}. Furthermore, these kinds of abstractions usually assume that the planning domain model does not contain actions with conditional effects. If that is not the case, the planning domain model can be re-compiled in order to satisfy this assumption. On the one hand, this implies that abstractions can be computed directly from the implicit definition of the planning task without the need of explicitly representing the transition graph, consequently providing a significant computational advantage. On the other hand, it constraints us to problems that are not heavy on conditional effects, otherwise the number of actions would scale with the state space size and a (P)PDDL~\cite{ghallab1998pddl, younes2004ppddl1} representation of the task may not be feasible in practice (\eg Atari $2600$ Games~\cite{bellemare2013arcade}). 

Abstractions are not confined only to the Planning field; they have, indeed, concurrently emerged in the context of discounted Markov Decision Processes (MDPs)~\cite{puterman2014markov}. A discounted MDP is the most commonly used model in the context of Reinforcement Learning (RL)~\cite{sutton2018reinforcement}, a field that tackles sequential decision making problems analogous to the ones addressed by AI Planning. Here, abstractions have the same overarching objective as in Planning, \ie  reducing the state space size of the original problem to transform the search for a solution into an easier task. However, techniques in this field commonly work on the aggregated state space producing a solution to be used directly on the original problem. This is in contrast \wrt Planning abstractions that, instead, usually guide the search for a solution through the computation of a possibly admissible heuristic function. Furthermore, stochastic processes over the aggregated state space are Partially Observable MDPs (POMDPs)~\cite{bai2016markovian} and a Markovian approximation of a non-Markovian process may result in poor performance. There are three abstraction frameworks within discounted MDPs: Weighting Function Abstractions (WFAs)~\cite{li2006towards}; Abstract Robust MDPs (ARMDPs)~\cite{petrik2014raam}; and Abstract Bounded Parameter MDPs (ABPMDPs)~\cite{givan2000bounded}. These techniques make no assumption on the aggregation function, but build abstractions via leveraging the transition and reward functions of the original task. This implies the utmost flexibility in defining how states of the original problem get aggregated into abstract states, but it is computationally intensive. Finally, the abstraction frameworks for discounted MDPs do not make any assumption about the conditional effects of actions. This, as already mentioned, comes at the cost of increased computational complexity, even though it allows more expressive power.

Considering the above discussion, there are many commonalities and differences for the same concept between the two fields. Therefore, in this paper, we will relate and unify projection abstractions in Planning (Classical and Probabilistic) with abstractions as they are conceived for discounted MDPs. Given a Planning abstraction, we will show how to obtain it under the lenses of the three abstraction frameworks we mentioned above, making explicit the role of the various different assumptions. This also constitutes a first step towards bridging the gap between AI Planning and RL in the context of abstractions.

\section{Background}\label{sec:background}
A common way tackle sequential decision making problems is through Planning techniques. A STRIPS~\cite{fikes1971strips} Planning task is defined as a tuple $P = \langle F, \mathcal{A}, I, G, C\rangle$, where $F$ is the set of propositions, $\mathcal{A}$ is the set of actions, $I \subseteq F$ is the initial state, $G \subseteq F$ represents the goal to be achieved, 
and $C: \mathcal{A} \rightarrow \mathbb{N}_1^+$ is the function associating a cost to each action. Every action $a \in \mathcal{A}$ is defined by a set of preconditions $\textsc{pre}\xspace(a)$ and effects $\textsc{eff}_a\xspace(e)$ with $e \sim P_a(p_{a,1},\dots,p_{a,n_a})$ a categorical distribution over the $n_a$ possible effects associated to action $a$. If $n_a = 1$ for each $a$ then we are in the Classical Planning setting and we drop the dependency on $e$ referring to the effects as $\textsc{eff}\xspace(a)$. The effects of an action are split into $\textsc{del}_a\xspace(e)$ and $\textsc{add}_a\xspace(e)$, and, in a set-based representation, given a state $s\in \mathcal{S}$ the next state $s' = (s\setminus\textsc{del}_a\xspace(e))\cup\textsc{add}_a\xspace(e)$. A binary vector representation of the state can be constructed as well using the propositions in $F$ as components of a state $s\in \mathcal{S}$ (overloading the notation for simplicity's sake, it will be clear when referring to one or the other from the context). In this case, we will represent the effects of an action $a$ as $\textsc{eff}_a\xspace(e) = f_a(e)$ in the case of Probabilistic Planning and $\textsc{eff}\xspace(a) = f(a)$ in the case of Classical Planning. In both formulations an action is applicable in state $s$ iff $\textsc{pre}\xspace(a)$ is satisfied in $s$. In the binary vector representation, applying action $a$ in state $s$,  we obtain $s' = s + \textsc{eff}_a\xspace(e)$, where add effects are intended as adding $1$ to the respective component of $s$ and delete effects subtract $1$ instead.\footnote{It is not allowed to the state space components to go below $0$ or above $1$. This can be formalized introducing a clamping function. Since $\alpha(clamp(s)) = clamp(\alpha(s))$, where $\alpha$ is a projection, it does not affect the derivations in this work, hence it has not been formalized into the notation to avoid cluttering.} A policy $\pi$ is a decision rule telling us what action to execute given the state we are in (it may be stochastic or deterministic). It constitutes a solution to the planning problem $P$ if, starting from the initial state, it reaches the goal with probability $1$ as the interactions tend to infinity. The policy is optimal if its expected cost is minimal. If $P$ is deterministic, then the concept of plan suffices. A plan is defined as a sequence of actions $\varpi = (a_1, \dots, a_n)$ and it is applicable in state $s_0$ if there exists a sequence of states $(s_1, \dots, s_n)$ such that $a_t$ is applicable in $s_{t-1}$ and $s' = (s\setminus\textsc{del}\xspace(a))\cup\textsc{add}\xspace(a)$ (or $s_t = s_{t-1} + f(a_t)$ in a binary vector representation). A plan is a solution to $P$ if it is applicable in $I$ and $G$ is satisfied in $s_n$. Furthermore, it is optimal if its cost, defined as $C(\varpi) = \sum_{t=1}^n C(a_t)$, is minimal among all alternative plans costs.

Another formalism to model sequential decision making problems is represented by discounted MDPs. A discounted MDP is defined by a tuple $\mathcal{M} = (\mathcal{S}, \mathcal{A}, \mathcal{T}, \mathcal{R}, \gamma, \rho)$, where $\mathcal{S}$ is the state space, $\mathcal{A}$ is the action space, $\mathcal{T}(s'|s, a)$ is the Markovian state transition function (a probability distribution over the next state given that we applied action $a$ in state $s$), $\mathcal{R}$ is the reward function, $\gamma \in [0, 1)$ is the discount factor, and $\rho$ is the initial state distribution. An optimal solution, in this setting, is represented by a policy $\pi$ maximizing the expected cumulative discounted reward $\mathbb{E}_{\tau \sim p_{\pi}}[\mathcal{R}(\tau)]$, where $\mathcal{R}(\tau) = \sum_{t=0}^\infty \gamma^t\mathcal{R}(s_t, a_t)$, $\tau=\langle s_0, a_0, s_1, \dots\rangle$ is a trajectory, and $p_{\pi}$ is the distribution over trajectories induced by the policy $\pi$, the transition function $\mathcal{T}$, and the initial state distribution $\rho$.

In both the above described models, state-space abstractions may be formalized  via an aggregation function $\alpha: \mathcal{S}\rightarrow\mathcal{\Bar{S}}$, where $\mathcal{S}$ is the state space of the problem at hand, $\mathcal{\Bar{S}}$ is the abstract state space, and $|\mathcal{\Bar{S}}|<|\mathcal{S}|$. The reduction in the state space size is what is supposed to give the computational advantage.

\subsection{Abstractions in Planning}\label{sec:background_planning}

Although there exist more general frameworks to build planning abstractions such as Cartesian abstractions~\cite{seipp2018counterexample} or Merge-and-shrink~\cite{helmert2007flexible}, in this work we will focus on projections, which form the basis for Pattern Database (PDB) heuristics~\cite{culberson1998pattern,edelkamp2002symbolic}.\footnote{We formalize projection abstractions and PDBs in terms of STRIPS due to easiness of exposition and derivation of results.
However, a multi-valued state variable encoding can represent PDBs much more compactly.}
\begin{definition}[Projection Abstractions]
    Let $P = \langle F, \mathcal{A}, I, G, C\rangle$ be a planning task. 
    Let $\Bar{F} \subseteq F$ be a pattern.
    A projection $\Bar{F}$ of the task $P$ generates an abstract task $\Bar{P} = \langle \Bar{F}, \Bar{A}, \Bar{I}, \Bar{G} \rangle$ where:
    \begin{itemize}
        \item $\Bar{I} = I \cap \Bar{F}$
        \item $\Bar{G} = G \cap \Bar{F}$
        \item $\Bar{A} = \bigcup_{a \in \mathcal{A}} \langle \textsc{pre}\xspace(a) \cap \Bar{F}, \textsc{eff}\xspace(a) \cap \Bar{F} \rangle$
    \end{itemize}
\end{definition}

In the case of a binary vector representation, the above intersection operation that gives us the abstraction can be substituted with a projection $\alpha(s) = \Bar{s}$, where $\Bar{s}$ is produced by keeping only the components within the pattern $\Bar{F}$.

PDBs are obtained by abstracting all but a part of the problem (the \emph{pattern}), yielding a problem that is small enough to be solved optimally for every state.
The results are computed offline and stored in a table called the \emph{pattern database}.
During search, the heuristic estimate of the given state is computed by accessing the value of the abstract state in the \emph{pattern database}.
Heuristics estimates from different abstractions can be combined by taking their maximum, or, under certain conditions, their sum~\cite{haslum2007domain}.

\subsection{Abstractions in MDPs}\label{sec:background_mdps}
We will now review the main abstraction frameworks for discounted MDPs: WFAs, ARMDPs, and ABPMDPs.

Given an aggregation function $\alpha: \mathcal{S}\rightarrow\mathcal{\Bar{S}}$ that arbitrarily aggregates states of the original problem into abstract states, a set-valued function $\alpha^{-1}:\Bar{S} \rightarrow 2^{\mathcal{S}}$ defining the abstract class for any abstract state $\Bar{s}\in\Bar{S}$ as $\alpha^{-1}(\Bar{s}) = \{s\in \mathcal{S}: \alpha(s) = \Bar{s}\}$, and a probability distribution $\omega_{\Bar{s}}: \alpha^{-1}(\Bar{s}) \rightarrow [0, 1]$ for any abstract state $\Bar{s}$, then:
\begin{restatable}[Weighting Function Abstractions]{definition}{Weighting Function Abstractions}\label{def:wfa}
Given an MDP $\mathcal{M} = (\mathcal{S}, \mathcal{A}, \mathcal{T}, \mathcal{R}, \gamma, \rho)$, a  WFA is a new MDP $\mathcal{M}_\omega=(\Bar{\mathcal{S}}, \mathcal{A}, \mathcal{\Bar{T}}_\omega, \mathcal{\Bar{R}}_\omega, \gamma)$ such that the transitions and the reward functions are defined in the following way \wrt $\mathcal{M}$:
\begin{align*}
    \mathcal{\Bar{T}}_\omega(\Bar{s}'|\Bar{s}, a) &= \sum_{s\in \alpha^{-1}(\Bar{s})}\omega_{\Bar{s}}(s)\sum_{s'\in \alpha^{-1}(\Bar{s}')}\mathcal{T}(s'|s, a)\\
    \mathcal{\Bar{R}}_\omega(\Bar{s}, a) &=  \sum_{s\in \alpha^{-1}(\Bar{s})}\omega_{\Bar{s}}(s)\mathcal{R}(s, a)
\end{align*}
\end{restatable}
A WFA performs a weighted average of all the transition probabilities (or the rewards, respectively) associated to the states in the abstract class represented by $\Bar{s}$. The weighting function is the same for all the actions applicable in $\Bar{s}$. For what concerns ARMDPs, we have to define a richer weighting function $\xi_{\Bar{s}, a}:\alpha^{-1}(\Bar{s}) \rightarrow [0,1]$, then:

\begin{restatable}[Abstract Robust MDPs]{definition}{Abstract Robust MDPs}\label{def:armdps}
Given an MDP $\mathcal{M} = (\mathcal{S}, \mathcal{A}, \mathcal{T}, \mathcal{R}, \gamma, \rho)$, an ARMDP is a new MDP $\mathcal{M}_\xi=(\Bar{\mathcal{S}}, \mathcal{A}, \mathcal{\Bar{T}}_\xi, \mathcal{\Bar{R}}_\xi, \gamma)$ such that the transition and reward functions are defined as:
\begin{align*}
    \mathcal{\Bar{T}}_\xi(\Bar{s}'|\Bar{s}, a) &= \sum_{s\in \alpha^{-1}(\Bar{s})}\xi_{\Bar{s}, a}(s)\sum_{s'\in \alpha^{-1}(\Bar{s}')}\mathcal{T}(s'|s, a)\\
    \mathcal{\Bar{R}}_\xi(\Bar{s}, a) &=  \sum_{s\in \alpha^{-1}(\Bar{s})}\xi_{\Bar{s}, a}(s)\mathcal{R}(s, a)
\end{align*}
\end{restatable}
Thanks to a more expressive weighting function that now will be different for any action applicable in an abstract state, ARMDPs can provide better approximations of the original MDP. Finally:

\begin{restatable}[Abstract Bounded Parameter MDPs]{definition}{Abstract Bounded Parameter MDPs}\label{def:abpmdps}
An ABPMDP is a family of MDPs $\mathcal{M}_I=(\Bar{\mathcal{S}}, \mathcal{A}, \mathcal{\Bar{T}}_I, \mathcal{\Bar{R}}_I, \gamma)$ such that the transition and reward functions are defined as:
\begin{align*}
    \mathcal{\Bar{T}}_I(\Bar{s}'|\Bar{s}, a) &= \Big[\min_{s\in \alpha^{-1}(\Bar{s})}\sum_{s'\in \alpha^{-1}(\Bar{s}')}\mathcal{T}(s'|s, a),\\ &\quad\qquad\qquad\max_{s\in \alpha^{-1}(\Bar{s})}\sum_{s'\in \alpha^{-1}(\Bar{s}')}\mathcal{T}(s'|s, a)\Big]\\
    \mathcal{\Bar{R}}_I(\Bar{s}, a) &= \left[\min_{s\in \alpha^{-1}(\Bar{s})}\mathcal{R}(s, a), \max_{s\in \alpha^{-1}(\Bar{s})}\mathcal{R}(s, a) \right]
\end{align*}
\end{restatable}

This last abstraction framework provides a family of possible abstract MDPs selecting transition and reward functions within the specified intervals. Solving this family of abstractions means searching for the optimal policies in the least and most favorable MDP belonging to the family. Finally, as shown by~\citealp{congeduti2022cross}, $\text{WFAs}\subseteq\text{ARMDPs}\subseteq\text{ABPMDPs}$.

\section{The Relation Between Abstractions in Classical Planning and MDPs}\label{sec:abstractions}
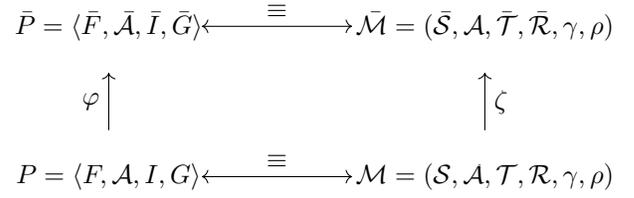
\begin{figure}
    \vspace{-0.6in}
    \centering
    \begin{tikzpicture}
        \node[shape=circle,draw=white, minimum size=70pt, yscale=0.5, draw opacity=0] (A) at (-2.5,-2) {};
        \node[shape=circle,draw=white, minimum size=70pt, yscale=0.5, draw opacity=0] (B) at (-2.5,0) {};
        \node[shape=circle,draw=white, minimum size=100pt, yscale=0.35, draw opacity=0] (C) at (2.5,-2) {};
        \node[shape=circle,draw=white, minimum size=100pt, yscale=0.35, draw opacity=0] (D) at (2.5,0) {};
        \node[shape=circle,draw=white] (AA) at (-2.5,-2) {$P = \langle F, \mathcal{A}, I, G\rangle$};
        \node[shape=circle,draw=white, draw opacity=0] (BB) at (-2.5,0) {$\Bar{P} = \langle \Bar{F}, \mathcal{\Bar{A}}, \Bar{I}, \Bar{G}\rangle$};
        \node[shape=circle,draw=white, draw opacity=0] (CC) at (2.5,-2) {$\mathcal{M} = (\mathcal{S}, \mathcal{A}, \mathcal{T}, \mathcal{R}, \gamma, \rho)$};
        \node[shape=circle,draw=white, draw opacity=0] (DD) at (2.5,0) {$\mathcal{\Bar{M}} = (\mathcal{\Bar{S}}, \mathcal{A}, \mathcal{\Bar{T}}, \mathcal{\Bar{R}}, \gamma, \rho)$};
        \path [->] (A) edge node[left] {$\varphi$} (B);
        \path [->] (C) edge node[right] {$\zeta$} (D);
        \path [<->] (A) edge node[above] {$\equiv$} (C);
        \path [<->] (B) edge node[above] {$\equiv$} (D);
    \end{tikzpicture}
    \vspace{-0.6in}
    \caption{Relationship between Planning and MDP \wrt the original problem description and the abstraction. The functions $\varphi$ and $\zeta$ transform a problem into its abstraction in the planning and MDP domains, respectively.}
    \label{fig:planning2mdp_equivalence}
\end{figure}
After introducing the necessary background, we can turn our attention to the relationship between abstractions in Classical Planning and MDPs. Starting from two equivalent problems, one represented as a Classical Planning task and another one as an MDP, and given an abstraction over the planning task, we will discuss how to obtain the equivalent abstraction, if feasible, in the context of the MDP formalism under the lenses of three different abstraction frameworks: WFAs, ARMDPs, and ABPMDPs  (see Figure~\ref{fig:planning2mdp_equivalence}).

If there are no conditional effects in our planning domain, then given a starting state $s$ where an action $a$ is applicable, the next state $s'$ will have the following form: $s' = s + f(a)$, where $f(a)$ represents the effects of action $a$. They are state independent due to the absence of conditional effects. Additionally, the abstraction in the planning domain may be defined through a projection $\alpha$ over the state space. This implies that $\alpha(s') = \alpha(s+f(a)) = \alpha(s) + \alpha(f(a))$. Therefore, the following property holds:

\begin{restatable}[Absence of ambiguity in the abstraction transition graph]{proposition}{Ambiguity Absence in the Abstraction Transition Graph}\label{prop:ambiguity_absence}
$\forall~s_1,s_2 \in \alpha^{-1}(\Bar{s})$ such that $a$ is applicable in $s_1$ and $s_2$, then $\alpha(s'_{s_1,a}) = \alpha(s'_{s_2,a}) = \Bar{s}'$.
\end{restatable}

\begin{proof}
    Since $s_1,s_2 \in \alpha^{-1}(\Bar{s})$, then $\alpha(s_1) = \alpha(s_2) = \Bar{s}$. Additionally, $s'_{s_1,a} = s_1 +f(a)$ and $s'_{s_2,a} = s_2 +f(a)$. This implies:
    \begin{align*}
        \alpha(s'_{s_1,a}) &= \alpha(s_1)+\alpha(f(a)) = \Bar{s} +\alpha(f(a)) \\ 
        &= \alpha(s_2) + \alpha(f(a)) = \alpha(s'_{s_2,a}) = \Bar{s}'.
    \end{align*}
\end{proof}

The above property avoids non-deterministic transitions in the abstract transition graph, where, starting from a state $\Bar{s}$ and applying an action $a$, we may move both to $\Bar{s}'$ and $\Bar{s}''$. An example is reported in Figure~\ref{fig:ambiguity}. While the property reported in Proposition~\ref{prop:ambiguity_absence} is always assumed in the context of abstractions for Classical Planning, it is never considered in the context of abstractions for MDPs. It is usually assumed that all actions are applicable in any state. In order to grant this property, the MDP can be recompiled into an equivalent one, where each action that has conditional effects gets remapped into a set of different actions, one per each conditional effect. In the worst case, the recompiled MDP will have $|\mathcal{A}||\mathcal{S}|$ actions. However, in the context of this paper, the MDP under analysis already satisfies this property, because it is equivalent to the starting planning problem. This equivalence implies that, in any given state of the MDP, only a subset of the action space will be applicable.

%
%
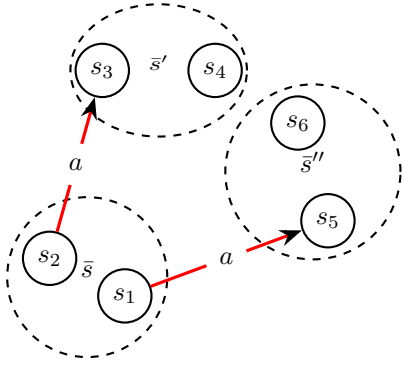
\begin{figure}[!tbh]
    \centering
    \begin{tikzpicture}
        \begin{scope}[every node/.style={circle,thick,draw}]
            \node (A) at (0,0) {$s_1$};
            \node (B) at (-1,0.5) {$s_2$};
           
            \node (C) at (-0.3,3) {$s_3$};
            \node (D) at (1.2, 3) {$s_4$};
            
            \node (E) at (2.7,1) {$s_5$};
            \node (F) at (2.3,2.3) {$s_6$};

            \node (S12) [rotate=0][draw,dashed,inner sep=0pt, circle,yscale=1,fit={(A)(B)}] {$\Bar{s}$};
            \node (S1N) [rotate=0][draw,dashed,inner sep=0pt, circle,yscale=0.75,fit={(C)(D)}] {$\Bar{s}'$};
            \node (S2N) [rotate=0][draw,dashed,inner sep=0pt, circle,yscale=1,fit={(E)(F)}] {$\Bar{s}''$};
        \end{scope}

        \begin{scope}[>={Stealth[black]},
                      every node/.style={fill=white,circle},
                      every edge/.style={draw=red,very thick}]
            \path [->] (B) edge node {$a$} (C);
            \path [->] (A) edge node {$a$} (E);
        \end{scope}
    \end{tikzpicture}
    \caption{Non-deterministic transitions induced in the abstraction's transition graph by actions with conditional effects.}
    \label{fig:ambiguity}
\end{figure}

In order to obtain an abstraction equivalent to the one in the Classical Planning context, we need to require the abstraction to be connection preserving and deterministic. Connection preserving intuitively means that if there is a connection between two different abstract classes in the original MDP then there is a connection between the relative abstract states in the abstraction through the same action. Let us dive into a formal definition:

\begin{restatable}[Connection Preserving]{definition}{Connection Preserving Abstraction}\label{def:connection_preserving}
Given and MDP $\mathcal{M}$, and $C^a_{\Bar{s}, \Bar{s}'}=\{s: s\in \alpha^{-1}(\Bar{s}), \sum_{s'\in \alpha^{-1}(\Bar{s}')}\mathcal{T}(s'|s, a)>0\}$, an abstraction is connection preserving whenever $C^a_{\Bar{s}, \Bar{s}'}\neq \emptyset \iff \mathcal{\Bar{T}}(\Bar{s}'|\Bar{s},a)>0$ and $C^a_{\Bar{s}, \Bar{s}'} = \emptyset \iff \mathcal{\Bar{T}}(\Bar{s}'|\Bar{s},a) = 0$.
\end{restatable}
The above property is of interest for abstractions in general, because it produces an abstraction that faithfully represents connections between abstract classes of the underlying MDP.
\begin{restatable}[Deterministic Abstraction]{definition}{Deterministic Abstraction}\label{def:deterministic}
An abstraction is deterministic if and only if: 
\begin{equation*}
    \mathcal{\Bar{T}}(\Bar{s}'|\Bar{s},a) =     \begin{cases}
      1 \\
      0
    \end{cases}
\end{equation*}
\end{restatable}
All the three different frameworks to construct abstractions for MDPs are connection preserving under some constraints over $\omega$, $\xi$, or $I$. 
\begin{restatable}[Connection Preserving WFAs]{proposition}{Connection Preserving WFAs}\label{prop:connection_preserving_wfa}
WFAs are connection preserving whenever $\forall \Bar{s}, \Bar{s}', a$ such that $C_{\Bar{s}, \Bar{s}'}^a\neq\emptyset$ there exists $s \in C_{\Bar{s}, \Bar{s}'}^a$ for which $\omega_{\Bar{s}}(s)>0$.
\end{restatable}

\begin{proof}
    $\implies$ Assumption: $C_{\Bar{s}, \Bar{s}'}^a\neq\emptyset$. Furthermore, by assumption of the proposition, it exists $s \in C_{\Bar{s}, \Bar{s}'}^a$ for which $\omega_{\Bar{s}}(s)>0$. For this $s \in C_{\Bar{s}, \Bar{s}'}^a$ it also holds that $\sum_{s'\in \alpha^{-1}(\Bar{s}')}\mathcal{T}(s'|s, a)>0$ by definition of $C_{\Bar{s}, \Bar{s}'}^a$, hence $\mathcal{\Bar{T}}(\Bar{s}'|\Bar{s}, a)>0$.

    $\impliedby$ Assumption: $\mathcal{\Bar{T}}(\Bar{s}'|\Bar{s}, a)>0$. Hence, $\sum_{s\in \alpha^{-1}(\Bar{s})}\omega_{\Bar{s}}(s)\sum_{s'\in \alpha^{-1}(\Bar{s}')}\mathcal{T}(s'|s, a)>0$ that implies it exists $s\in\alpha^{-1}(\Bar{s})$ such that $\omega_{\Bar{s}}(s)>0\wedge\sum_{s'\in \alpha^{-1}(\Bar{s}')}\mathcal{T}(s'|s, a)>0$. This means that $C_{\Bar{s}, \Bar{s}'}^{a}\neq\emptyset$.

    $\implies$ Assumption: $C_{\Bar{s}, \Bar{s}'}^a = \emptyset$. Therefore, by definition of $C_{\Bar{s}, \Bar{s}'}^a$, $\forall s \in \alpha^{-1}(\Bar{s})$ then $\sum_{s'\in \alpha^{-1}(\Bar{s}')}\mathcal{T}(s'|s, a)=0$.

    $\impliedby$ Assumption: $\mathcal{\Bar{T}}(\Bar{s}'|\Bar{s}, a)=0$. This implies that $\forall s \in \alpha^{-1}(\Bar{s})$ either $\sum_{s'\in \alpha^{-1}(\Bar{s}')}\mathcal{T}(s'|s, a)=0$ or if $\sum_{s'\in \alpha^{-1}(\Bar{s}')}\mathcal{T}(s'|s, a)>0$ then $\omega_{\Bar{s}}(s) = 0$. The second case cannot happen because it implies that $C_{\Bar{s}, \Bar{s}'}^a \neq \emptyset$, hence it had to exist $s\in C_{\Bar{s}, \Bar{s}'}^a$ such that $\omega_{\Bar{s}}(s)>0$ meaning that $\mathcal{\Bar{T}}(\Bar{s}'|\Bar{s}, a)>0$, contradicting the assumption we started from. Therefore, $C_{\Bar{s}, \Bar{s}'}^a=\emptyset$.
\end{proof}

\begin{restatable}[Connection Preserving ARMDPs]{proposition}{Connection Preserving ARMDPs}\label{prop:connection_preserving_armdps}
ARMDPs are connection preserving whenever $\forall \Bar{s}, \Bar{s}', a$ such that $C_{\Bar{s}, \Bar{s}'}^a\neq\emptyset$ then exists $s \in C_{\Bar{s}, \Bar{s}'}^a$ for which $\xi_{\Bar{s},a}(s)>0$.
\end{restatable}
\begin{proof}
    The proof is analogous to the one above substituting $\omega$ for $\xi$.
\end{proof}

\begin{restatable}[Connection Preserving ABPMDPs]{proposition}{Connection Preserving ABPMDPs}\label{prop:connection_preserving_abpmdps}
ABPMDPs are connection preserving whenever $\forall \Bar{s}, \Bar{s}', a$ such that $C_{\Bar{s}, \Bar{s}'}^a\neq\emptyset$ then if $\min_{s\in \alpha^{-1}(\Bar{s})}\sum_{s'\in \alpha^{-1}(\Bar{s}')}\mathcal{T}(s'|s, a) = 0$, it must be excluded from the interval.
\end{restatable}

\begin{proof}
    $\implies$ Assumption: $C_{\Bar{s}, \Bar{s}'}^a\neq\emptyset$. Then, by definition of $C_{\Bar{s}, \Bar{s}'}^a$, it exists $s \in \alpha^{-1}(\Bar{s})$ such that $\sum_{s'\in \alpha^{-1}(\Bar{s}')}\mathcal{T}(s'|s, a)>0$, implying $\max_{s\in \alpha^{-1}(\Bar{s})}\sum_{s'\in \alpha^{-1}(\Bar{s}')}\mathcal{T}(s'|s, a)>0$. By the assumption of the proposition $0$ will be excluded as a lower bound of the interval,  hence $\mathcal{\Bar{T}}(\Bar{s}'|\Bar{s}, a)>0$.

    $\impliedby$ Assumption $\mathcal{\Bar{T}}(\Bar{s}'|\Bar{s}, a)>0$. This only happens if $\max_{s\in \alpha^{-1}(\Bar{s})}\sum_{s'\in \alpha^{-1}(\Bar{s}')}\mathcal{T}(s'|s, a)>0$, that, in turn, implies $C_{\Bar{s}, \Bar{s}'}^a\neq\emptyset$.

    $\implies$ Assumption: $C_{\Bar{s}, \Bar{s}'}^a=\emptyset$. Then it does not exists $s\in \alpha^{-1}(\Bar{s})$ such that $\sum_{s'\in \alpha^{-1}(\Bar{s}')}\mathcal{T}(s'|s, a)>0$. This further implies that $\mathcal{\Bar{T}}(\Bar{s}'|\Bar{s}, a)=[0,0] =0$.

    $\impliedby$ Assumption $\mathcal{\Bar{T}}(\Bar{s}'|\Bar{s}, a)=0$. If the interval would be of the following form: $\mathcal{\Bar{T}}(\Bar{s}'|\Bar{s}, a)=[0,c]$, with $c>0$ we would discard zero according to the assumption of the proposition. The last case that remains is $\mathcal{\Bar{T}}(\Bar{s}'|\Bar{s}, a)=[0,0] =0$ that implies it does not exists $s \in \alpha^{-1}(\Bar{s})$ such that $\sum_{s'\in \alpha^{-1}(\Bar{s}')}\mathcal{T}(s'|s, a)>0$, further implying $C_{\Bar{s}, \Bar{s}'}^a=\emptyset$.
\end{proof}

\subsection{Weighting Function Abstractions}\label{sec:wfa_mapping}
For what concerns WFAs, other than requiring the condition stated in Proposition \ref{prop:connection_preserving_wfa}, we need to require the obtained abstraction to be deterministic, see Definition~\ref{def:deterministic}. Since there may be actions in any state of the original MDP that are not applicable (\ie the transition function has zero probability for all next states), then for the abstraction to be deterministic we have to further require that:
\begin{equation}
    \sum_{s\in C_{\Bar{s}, \Bar{s}'}^a}\omega_{\Bar{s}}(s) = 1 ~\forall~ a, \Bar{s}, \Bar{s}' ~\text{with}~ \Bar{s}\neq\Bar{s}'.\label{eq:wfa_remapping_cond}
\end{equation}
The condition $\Bar{s}\neq\Bar{s}'$ is added to remove self loops in the abstraction that are useless. In the context of Planning, we can remove all the actions whose effects are entirely annihilated by the aggregation function. 

Unfortunately, there are cases were the condition reported in Eq.~\eqref{eq:wfa_remapping_cond} cannot be satisfied. Let us imagine an abstract state $\Bar{s}=\{s_1, s_2, s_3\}$ with action $a$ applicable in $s_1$ and $s_2$, action $b$ applicable in $s_2$, and action $c$ applicable in $s_1$. Eq. \eqref{eq:wfa_remapping_cond} yields the following system of constraints:
\begin{align*}
    \begin{cases}
        \omega_{\Bar{s}}(s_1) + \omega_{\Bar{s}}(s_2) = 1\\
        \omega_{\Bar{s}}(s_2) = 1\\
        \omega_{\Bar{s}}(s_1) = 1 \\
    \end{cases}
\end{align*}
that has no solution.
We report a more complete example of such a contradiction in Figure \ref{fig:degenerate}. Even if $\omega_{\Bar{s}}$ was a generic function and not a probability distribution over the abstract class, we would not be able to satisfy that system of constraints. In the context of Planning, this situation may happen only when the aggregation function abstracts away some preconditions of three actions applicable in different states of the same abstract class. In the case where the preconditions of the actions in an abstract class are not affected by the projection function $\alpha$, we do not have any problem because the actions will be applicable in any $s \in \alpha^{-1}(\Bar{s})$, hence, the constraint in Eq. \eqref{eq:wfa_remapping_cond} can be satisfied. 

\begin{figure*}[!tbh]
    \centering
    \begin{tikzpicture}
        \begin{scope}[every node/.style={circle,thick,draw}]
            \node (A) at (-7,1) {$s_1$};
            \node (B) at (-5,1) {$s_2$};
            \node (C) at (-7,-1) {$s_3$};
            \node (D) at (-5,-1) {$s_4$};
            
            \node (E) at (-7,5) {$s_5$};
            \node (F) at (-7,3) {$s_6$};
            \node (G) at (-5, 5) {$s_7$};
            \node (H) at (-5,3) {$s_{8}$};
            
            \node (I) at (2,5) {$s_{9}$};
            \node (J) at (2, 3) {$s_{10}$};
            \node (K) at (0, 5) {$s_{11}$};
             \node (L) at (0, 3) {$s_{12}$};

            \node (M) at (2,1) {$s_{13}$};
            \node (N) at (2, -1) {$s_{14}$};
            \node (O) at (0, 1) {$s_{15}$};
            \node (P) at (0, -1) {$s_{16}$};

            \node (S12) [rotate=0][draw,dashed,inner sep=0pt, circle,yscale=0.9,fit={(A)(B)(C)(D)}] {$\Bar{s}_0$};
            \node (S1N) [rotate=0][draw,dashed,inner sep=0pt, circle,yscale=0.9,fit={(E)(F)(G)(H)}] {$\Bar{s}_1$};
            \node (S2N) [rotate=0][draw,dashed,inner sep=0pt, circle,yscale=0.9,fit={(L)(J)(I)(K)}] {$\Bar{s}_2$};
            \node (S2N) [rotate=0][draw,dashed,inner sep=0pt, circle,yscale=0.9,fit={(M)(N)(O)(P)}] {$\Bar{s}_3$};
        \end{scope}

        \begin{scope}[>={Stealth[black]},
                      every node/.style={fill=white,circle},
                      every edge/.style={draw=red,very thick}]
            \path [->] (F) edge node {$b$} (K);
            \path [->] (H) edge node {$h$} (N);
            \path [->] (F) edge node {$f$} (E);
            \path [->] (H) edge node {$f$} (G);
            \path [->] (F) edge[bend right=30] node {$a$} (C);
            \path [->] (H) edge[bend left=30] node {$a$} (D);

            \path [->] (O) edge node {$g$} (M);
            \path [->] (P) edge node {$g$} (N);
            \path [->] (M) edge[bend left=30] node {$d$} (I);
            \path [->] (N) edge[bend right=30] node {$d$} (J);
            \path [->] (O) edge[bend right=30] node {$d$} (K);
            \path [->] (P) edge[bend left=30] node {$d$} (L);

            \path [->] (J) edge[bend right=30] node {$c$} (H);
            \path [->] (L) edge[bend left=30] node {$c$} (F);
        \end{scope}
        \matrix [draw,below left] at (6.5,6) {
            \node [label=right:{$p_2$}] {$s_1$}; \\
            \node [label=right:{$p_2, p_3$}] {$s_2$}; \\
            \node [label=right:{$p_2, p_4$}] {$s_3$}; \\
            \node [label=right:{$p_2, p_3, p_4$}] {$s_4$}; \\
            \node [label=right:{$p_1, p_2$}] {$s_5$}; \\
            \node [label=right:{$p_1, p_2, p_4$}] {$s_6$}; \\
            \node [label=right:{$p_1, p_2, p_3$}] {$s_7$}; \\
            \node [label=right:{$p_1, p_2, p_3, p_4$}] {$s_8$}; \\
            \node [label=right:{$p_1, p_3$}] {$s_9$}; \\
            \node [label=right:{$p_1, p_3, p_4$}] {$s_{10}$}; \\
            \node [label=right:{$p_1$}] {$s_{11}$}; \\
            \node [label=right:{$p_1, p_4$}] {$s_{12}$}; \\
            \node [label=right:{$p_3$}] {$s_{13}$}; \\
            \node [label=right:{$p_3, p_4$}] {$s_{14}$}; \\
            \node [label=right:{}] {$s_{15}$}; \\
            \node [label=right:{$p_4$}] {$s_{16}$}; \\
        };
    \end{tikzpicture}
    \caption{An abstraction that cannot be represented through WFA. The aggregation function retains only $p_1$ and $p_2$, the original initial state is $p_4$, and the goal is $\neg p_1 \wedge p_2 \wedge p_3$. Action $a$ requires $\omega_{\Bar{s}_1}(s_6) + \omega_{\Bar{s}_1}(s_8) = 1$, action $b$ requires $\omega_{\Bar{s}_1}(s_6) = 1$, action $h$ requires $\omega_{\Bar{s}_1}(s_8) = 1$, that yield a contradiction.}
    \label{fig:degenerate}
\end{figure*}
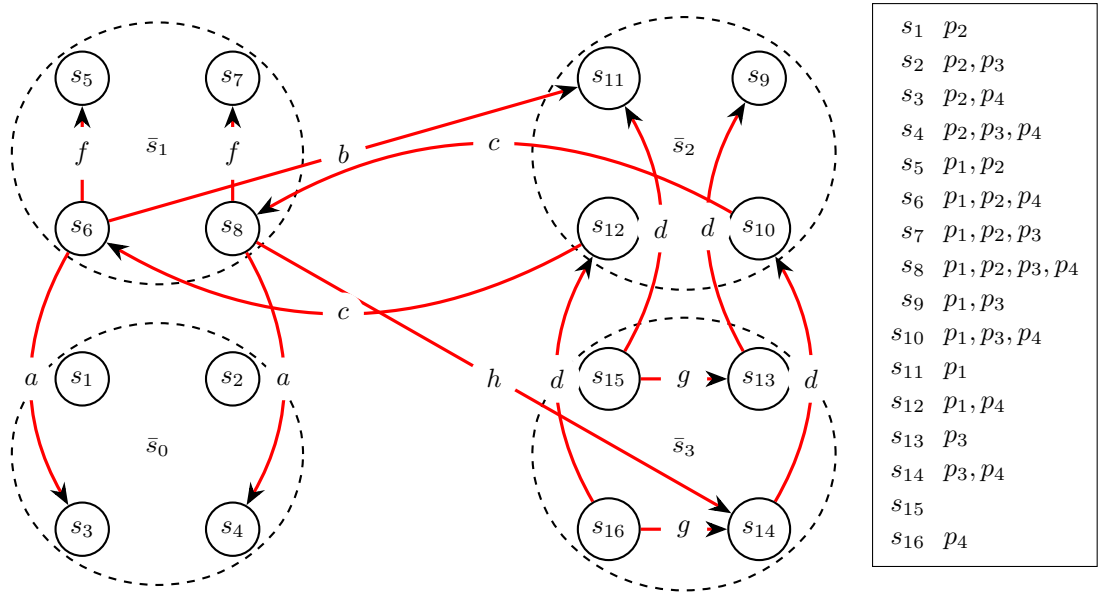

\subsection{Abstract Robust Markov Decision Processes}
In the previous section, we have seen how WFAs, in some cases, may not be able to preserve connectivity \wrt the original MDP. This was due to the fact that we were forced to use the same weighting function $\omega_{\Bar{s}}$ for multiple, different actions. In the context of ARMDPs, this is no longer a problem. Indeed, the requirement in Eq.~\eqref{eq:wfa_remapping_cond} becomes:
\begin{equation}
    \sum_{s\in C_{\Bar{s}, \Bar{s}'}^a}\xi_{\Bar{s}, a}(s) = 1 ~\forall~ a, \Bar{s}, \Bar{s}' ~\text{with}~ \Bar{s}\neq\Bar{s}'.\label{eq:armdps_remapping_cond}
\end{equation}
Having one weighting function per action allows us to avoid over-restrictive requirements on the weighting function itself as it was happening for WFAs.
Now, let us check that the condition reported in Eq.~\eqref{eq:armdps_remapping_cond} allows us to represent the planning abstraction. If $C_{\Bar{s}, \Bar{s}'}^a = \emptyset$ then $\Bar{\mathcal{T}}(\Bar{s}'|\Bar{s}, a) = 0$ because $\forall~s\in \alpha^{-1}(\Bar{s})$ then $\sum_{\Bar{s}' \in \alpha^{-1}(\Bar{s}')}\mathcal{T}(s'|s, a) = 0$. Otherwise, if $C_{\Bar{s}, \Bar{s}'}^a \neq \emptyset$ then $\sum_{\Bar{s}' \in \alpha^{-1}(\Bar{s}')}\mathcal{T}(s'|s, a) = 1 ~\forall~s\in C_{\Bar{s}, \Bar{s}'}^a$ due to the fact that the original MDP is deterministic. Now, we can conclude that:
\begin{align}
    \Bar{\mathcal{T}}(\Bar{s}'|\Bar{s}, a) = & \sum_{s\in\alpha^{-1}(\Bar{s})}\xi_{\Bar{s},a}(s)\sum_{s' \in \alpha^{-1}(\Bar{s}')}\mathcal{T}(s'|s,a) = \nonumber\\
    &\sum_{s\in C_{\Bar{s}, \Bar{s}'}^a}\xi_{\Bar{s},a}(s)\sum_{s' \in \alpha^{-1}(\Bar{s}')}\mathcal{T}(s'|s,a) = \label{eq:armdps_use_remapping_cond_1}\\
    &\sum_{s\in C_{\Bar{s}, \Bar{s}'}^a}\xi_{\Bar{s},a}(s) = 1\label{eq:armdps_use_remapping_cond_2}
\end{align}
In Eq.~\ref{eq:armdps_use_remapping_cond_1}, we have used the fact that for each $s \notin C_{\Bar{s}, \Bar{s}'}^a$ and $s \in \alpha^{-1}(\Bar{s})$, then $\sum_{\Bar{s}' \in \alpha^{-1}(\Bar{s}')}\mathcal{T}(s'|s, a) = 0$ and in Eq.~\eqref{eq:armdps_use_remapping_cond_2}, we have used $\sum_{\Bar{s}' \in \alpha^{-1}(\Bar{s}')}\mathcal{T}(s'|s, a) = 1 ~\forall~s\in C_{\Bar{s}, \Bar{s}'}^a$ together with \eqref{eq:armdps_remapping_cond}. Finally, there cannot exist another $\Bar{s}''\neq\Bar{s}'$ such that $\Bar{\mathcal{T}}(\Bar{s}''|\Bar{s}, a) \neq 0$. Otherwise, we would violate the condition required by Proposition~\ref{prop:ambiguity_absence}. To conclude, the above implies that the obtained abstraction is deterministic and connection preserving, hence equivalent to the planning abstraction. 
Finally, for what concerns the reward function, given an abstract class $\Bar{s}$ and an action $a$ such that $C^a_{\Bar{s}, \Bar{s}' \cap \mathcal{S}_G} = \{s: s\in \alpha^{-1}(\Bar{s})\wedge \sum_{s'\in \alpha^{-1}(\Bar{s}')\cap \mathcal{S}_G}\mathcal{T}(s'|s, a)>0\}\neq\emptyset$, we require:
\begin{equation}
    \sum_{s\in C_{\Bar{s}, \Bar{s}'\cap \mathcal{S}_G}^a}\xi_{\Bar{s}, a}(s) = 1 ~\text{with}~ \Bar{s}\neq\Bar{s}'.\label{eq:armdps_reward_cond}
\end{equation}
The above equation still satisfies Eq.~\eqref{eq:armdps_remapping_cond} and implies a unitary reward when reaching an abstract goal.~\footnote{It can also handle non-uniform action costs due to the fact that the weight function is specialized to the action.} An example can be seen in Figure~\ref{fig:degenerate} when executing action $a$ in $s_8$ brings us to a goal while executing it in $s_6$ does not. Hence, it is necessary to select $\xi_{\Bar{s}_1, a}(s_8) = 1$ so that $\mathcal{R}_\xi(\Bar{s}_1,a)=1$ as well. This further requirement is necessary only for those abstract classes that have states connected to some goals. If the aggregation function $\alpha$ does not affect the goal, and if applying an action $a$ in a state $s_1\in \alpha^{-1}(\Bar{s})$ the agent lands on $s_1' \in \mathcal{S}_G$, then all the states $s'\in\alpha^{-1}(\alpha(s'_1))$ belong to $\mathcal{S}_G$. This implies that $\alpha^{-1}(\alpha(s'_1))\cap\mathcal{S}_G = \alpha^{-1}(\alpha(s'_1))$ and the additional condition stated in Eq.~\ref{eq:armdps_reward_cond} is superfluous.

\subsection{Abstract Bounded Parameter Markov Decision Processes}
Since the family of abstractions represented by ABPMDPs contains ARMPDs, we can obtain an abstraction that is equivalent to the planning one. In order to do so, it is sufficient to select:
\begin{align*}
    \Bar{\mathcal{T}}(\Bar{s}'|\Bar{s}, a) &= \max_{s\in \alpha^{-1}(\Bar{s})}\sum_{s'\in \alpha^{-1}(\Bar{s}')}\mathcal{T}(s'|s, a)\\
    \Bar{\mathcal{R}}(\Bar{s}, a) &= \max_{s\in \alpha^{-1}(\Bar{s})}\mathcal{R}(s, a).
\end{align*}
The above definition preserves connectivity. Indeed, it satisfies the condition required by Proposition~\ref{prop:connection_preserving_abpmdps}. Furthermore, due to the fact that the underlying MDP is deterministic, the abstraction keeps on being deterministic choosing the maximum for the abstract transition function. The reward function $\Bar{\mathcal{R}}(\Bar{s}, a)$ will be $1$ if there exists an $s\in\alpha^{-1}(\Bar{s})$ such that $s$ through the action $a$ lands on a goal.~\footnote{It is worth to point out the fact that also ABPMDPs are able to handle non uniform costs for the actions, indeed, the definition of the aggregated reward is still action dependent.} We can conclude that projection abstractions in Classical Planning are positioned according to Figure~\ref{fig:abstraction_venn} \wrt abstractions in MDPs.
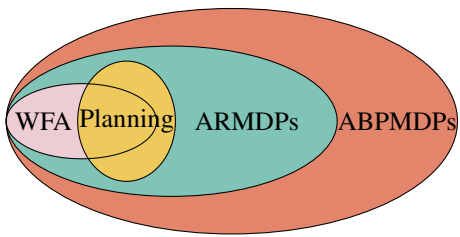
\begin{figure}
    \centering
    \def\wfa{(-2.5,0) ellipse (1cm and 0.5cm)}
    \def\armdp{(-1.3,0) ellipse (2.2cm and 1cm)}
    \def\abpmdp{(-0.5,0) ellipse (3cm and 1.5cm)}
    \def\planning{(-1.9, 0) ellipse (0.65cm and 0.8cm)}
    \begin{tikzpicture}
        \begin{scope}[transparency group]
        \begin{scope}[shift={(3cm,-5cm)}, fill opacity=1, blend mode=normal]
            \fill[PapayaOrange] \abpmdp;
            \fill[Teal] \armdp;
            \fill[PastelPink] \wfa;
            \fill[ButterYellow] \planning;
            
            \draw \abpmdp node at(1.7,0) {ABPMDPs};
            \draw \armdp node at(-0.3, 0) {ARMDPs};
            \draw \wfa node[left] {WFA};
            \draw \planning node {Planning};
        \end{scope}
        \end{scope}
    \end{tikzpicture}
    \caption{Relations among projection abstractions in Planning and abstarctions in MDPs.}
    \label{fig:abstraction_venn}
\end{figure}
\subsection{From Classical Planning to MDPs: an Example}\label{sec:example}
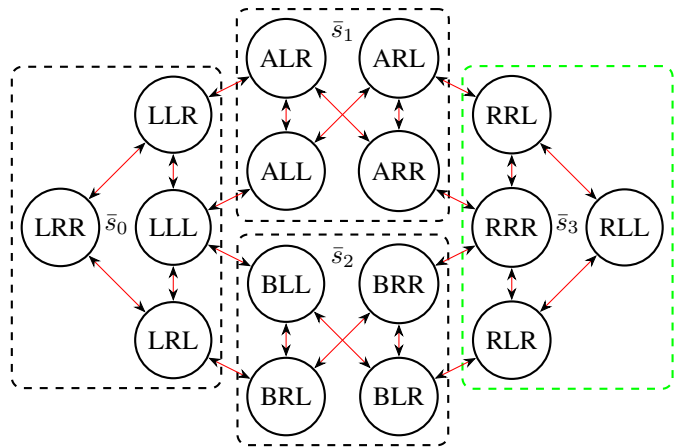
\begin{figure}[!tbh]
    \centering
    \begin{tikzpicture}
        \begin{scope}[every node/.style={circle,thick,draw}]
            \node (BRL) at (0,0) {BRL};
            \node (BLL) at (0,1.5) {BLL};
           
            \node (ALL) at (0,3) {ALL};
            \node (ALR) at (0, 4.5) {ALR};
            
            \node (BLR) at (1.5,0) {BLR};
            \node (BRR) at (1.5,1.5) {BRR};
            \node (ARR) at (1.5,3) {ARR};
            \node (ARL) at (1.5,4.5) {ARL};
           
            \node (RLR) at (3, 0.75) {RLR};
            \node (RRR) at (3, 2.25) {RRR};
            
            \node (RRL) at (3, 3.75) {RRL};
            \node (LRL) at (-1.5, 0.75) {LRL};
            \node (LLL) at (-1.5, 2.25) {LLL};
            \node (LLR) at (-1.5, 3.75) {LLR};
            
            \node (LRR) at (-3, 2.25) {LRR};
            \node (RLL) at (4.5, 2.25) {RLL};
            \node[label={[xshift=0cm, yshift=-2.5cm]$\Bar{s}_0$},draw=black,fit=(LRR) (LLR) (LLL) (LRL),rounded corners,dashed, style=rectangle] (fit) {};
            \node[label={[xshift=0cm, yshift=-0.7cm]$\Bar{s}_1$},draw=black,fit=(ARR) (ALR) (ALL) (ARL),rounded corners,dashed, style=rectangle] (fit) {};
            \node[label={[xshift=0cm, yshift=-0.7cm]$\Bar{s}_2$}, draw=black,fit=(BRR) (BLR) (BLL) (BRL),rounded corners,dashed, style=rectangle] (fit) {};
            \node[label={[xshift=0cm, yshift=-2.5cm]$\Bar{s}_3$},draw=green,fit=(RRR) (RLR) (RLL) (RRL),rounded corners,dashed, style=rectangle] (fit) {};
        \end{scope}

        \begin{scope}[>={Stealth[black]},            every edge/.style={draw=red}]
            \path [<->] (LRR) edge node {} (LLR);
            \path [<->] (LRR) edge node {} (LRL);
            \path [<->] (LRL) edge node {} (LLL);
            \path [<->] (LLR) edge node {} (LLL);
            \path [<->] (ARR) edge node {} (ALR);
            \path [<->] (ARR) edge node {} (ARL);
            \path [<->] (ARL) edge node {} (ALL);
            \path [<->] (ALR) edge node {} (ALL);
            \path [<->] (BRR) edge node {} (BLR);
            \path [<->] (BRR) edge node {} (BRL);
            \path [<->] (BRL) edge node {} (BLL);
            \path [<->] (BLR) edge node {} (BLL);
            \path [<->] (RRR) edge node {} (RLR);
            \path [<->] (RRR) edge node {} (RRL);
            \path [<->] (RRL) edge node {} (RLL);
            \path [<->] (RLR) edge node {} (RLL);
            \path [<->] (LLR) edge node {} (ALR);
            \path [<->] (LLL) edge node {} (ALL);
            \path [<->] (LLL) edge node {} (BLL);
            \path [<->] (LRL) edge node {} (BRL);
            \path [<->] (ARL) edge node {} (RRL);
            \path [<->] (ARR) edge node {} (RRR);
            \path [<->] (BRR) edge node {} (RRR);
            \path [<->] (BLR) edge node {} (RLR);
        \end{scope}
    \end{tikzpicture}
    \caption{The Logistic transition graph for the running example and its abstraction. A node represents a state. The first component of the state is the location of the package ($L,R,A,B$). The second component is the location of truck $A$ ($L,R$), and the last component is the location of truck $B$ ($L,R$). The dashed nodes are abstracted states merging together states that share the same location of the package. The green dashed node aggregates all the goals of the problem. A SAS+ representation has been used for space constraints. Any state (\eg $I=LRR$) can be represented in its equivalent propositional (STRIPS) version (\eg $I = 10000101$).}
    \label{fig:logistic}
\end{figure}

In order to make all the above concepts more concrete, in this section, we will consider an instance of a simplified Logistics domain~\cite{torralba2018aipl} in Classical Planning, fix an abstraction of it, and then obtain the same abstraction through ARMDPs. The considered simplified Logistics problem instance has $2$ locations, namely $L$ and $R$, two trucks, namely $A$ and $B$, and one package $P$. The package starts at location $L$ and needs to reach location $R$. The trucks are currently all at location $R$. A truck may move between locations. A package can be loaded on a truck if they are both in the same location or it can be unloaded in the location the loaded truck is at. In order to represent this problem, we have $8$ binary propositions, hence our state is:
\begin{align*}
    s = \Big(at(P, L), at(P, R), in(P, A), in(P, B), \\ at(A, L), at(A, R), at(B, L), at(B, R)\Big).
\end{align*}

In this example, we will consider a Planning abstraction having an aggregation function that disregards the last four components of the state vector, that can be interpreted as the trucks being everywhere. Hence:
\begin{align*}
    \Bar{s} = \alpha(s) = \Big(at(P, L), at(P, R), in(P, A), in(P, B)\Big).
\end{align*}
A representation of the transition graph for the above described problem and its abstraction is shown in Figure~\ref{fig:logistic}. The actions in PDDL are reported in Listing~\ref{lst:actions}. Their abstractions can be obtained by removing all the predicates of this form \textit{(at ?t ?l)} where \textit{t} is a truck and \textit{l} is a location. If we ground the actions with the objects of the problem (or the abstraction, respectively), we can exhaustively instantiate the transition graph.
Now, starting from the same transition graph represented as an MDP, we will use ARMDPs to obtain the same abstraction as the one we have in Classical Planning. This will be done satisfying the constraint imposed by Eq.~\ref{eq:armdps_remapping_cond}. We consider only the actions \textit{Load} and \textit{Unload} that correspond to edges between the abstract states in Figure~\ref{fig:logistic} because they are the only ones whose effects are not annihilated by the aggregation function. Given that $C_{\Bar{s}_0, \Bar{s}_1}^{Load_{L, A}}=\{LLR, LLL\}$, $C_{\Bar{s}_0, \Bar{s}_2}^{Load_{L, B}}=\{LRL, LLL\}$, $C_{\Bar{s}_1, \Bar{s}_3}^{Unload_{R, A}}=\{ARL, ARR\}$, $C_{\Bar{s}_2, \Bar{s}_3}^{Unload_{R, B}}=\{BLR, BRR\}$, $C_{\Bar{s}_3, \Bar{s}_1}^{Load_{R, A}}=\{RRL, RRR\}$, $C_{\Bar{s}_3, \Bar{s}_2}^{Load_{R, B}}=\{RLR, RRR\}$, $C_{\Bar{s}_1, \Bar{s}_0}^{Unload_{L, A}}=\{ALR, ALL\}$, and $C_{\Bar{s}_2, \Bar{s}_0}^{Unload_{L, B}}=\{BRL, BLL\}$, then we can satisfy Eq.~\ref{eq:armdps_remapping_cond} setting $\xi_{\Bar{s}, a}(s) = 0.5$ for any $a \in \{Load_{L, A}, Load_{L, B}, Load_{R, A}, Load_{R, B}, Unload_{L, A},\allowbreak Unload_{L, B}, Unload_{R, A}, Unload_{R, B}\}$, $s \in C_{\Bar{s}, \Bar{s}'}^a$, $\Bar{s}$ and $\Bar{s}'$ both in $\{\Bar{s}_0,\Bar{s}_1, \Bar{s}_2, \Bar{s}_3\}$.\footnote{$(Un)Load_{\{L, R\}, \{A, B\}}$ refers to the grounded actions.} Let us see an example with $\Bar{s} = \Bar{s}_0$, $a = Load_{L, A}$, $C_{\Bar{s}, \Bar{s}'}^a = C_{\Bar{s}_0, \Bar{s}_1}^{Load_{L, A}}$. In this case $\xi_{\Bar{s}_0, Load_{L, A}}(LLL)+\xi_{\Bar{s}_0, Load_{L, A}}(LLR) = 0.5 + 0.5 = 1$. Having two states in the abstract class $\alpha^{-1}(\Bar{s}_0)$ where $Load_{L, A}$ is applicable prevents the abstract transition function to be uniquely defined. Choosing $\xi_{\Bar{s}_0, Load_{L, A}}$ as above solve this issue, preserve the connectivity between $\Bar{s}_0$ and $\Bar{s}_1$, and yields a deterministic abstract transition function. Thus, we have an equivalent representation.

\begin{listing}[tb]%
\caption{Actions {\tt Simplified Logistics}.}%
\label{lst:actions}%
\begin{lstlisting}[language=PDDL]
(:action Load
    :parameters (?l - location ?p - package ?t - truck)
    :precondition (and (at ?p ?l) (at ?t ?l))
    :effect (and (in ?p ?t) not (at ?p ?l)))
(:action Unload
    :parameters (?l - location ?p - package ?t - truck)
    :precondition (and (in ?p ?t) (at ?t ?l))
    :effect (and not (in ?p ?t) (at ?p ?l)))
(:action Move
    :parameters (?l1 ?l2 - location ?t - truck)
    :precondition (at ?t ?l1)
    :effect (and (at ?t ?l2) not (at ?t ?l1)))
\end{lstlisting}
\end{listing}

\section{The Relation Between Abstractions in Probabilistic Planning and MDPs}
Recently, in the context of Probabilistic Planning~\cite{natarajan2022planning},~\citealp{klossner2021pattern} and~\citealp{klossner2021ssp} developed PDBs able to take into account the stochastic nature of the transition function characterizing an MDP. They do so under the goal-probability maximization~\cite{kolobov2011heuristic} and Stochastic Shortest Path (SSP)~\cite{bertsekas1991analysis} models (we will only consider SSPs in the context of this work, the abstraction is the same in both frameworks). 
The development of their heuristics hinges on the fact that the transition function of their abstraction, $\mathcal{\Bar{T}}(\Bar{s}'|\Bar{s}, a)$, does not depend on the choice of the representative state $s\in \alpha^{-1}(\Bar{s})$ where $a$ is applicable. However, this property finds its roots in the absence of conditional effects for the actions in our probabilistic planning domain model. To show this, let us generalize Proposition~\ref{prop:ambiguity_absence} to the Probabilistic Planning case. If there are no conditional effects in our probabilistic planning domain, then given a state $s$ where action $a$ is applicable, the next state $s'$ will have the following form: $s' = s + f_a(e)$, $e \sim P_a(p_{a,1},\dots,p_{a,n_a})$. Furthermore, observe that, $\forall s_1, s_2$ such that $s_1\neq s_2$ where $a$ is applicable, $\forall e$ effect of $a$, then $\mathcal{T}(s_1+f_a(e)|s_1,a)=\mathcal{T}(s_2+f_a(e)|s_2, a)$.
\begin{restatable}[Ambiguity Absence in the Abstraction Transition Graph for Probabilistic Planning]{proposition}{Absence of ambiguity in the Abstraction Transition Graph for Probabilistic Planning}\label{prop:ambiguity_absence_pplanning}
$\forall~s_1,s_2 \in \alpha^{-1}(\Bar{s})$ such that $a$ is applicable in $s_1$ and $s_2$, then, given any realization $e$ sampled from $P_a$, $\alpha(s'_{s_1,a}) = \alpha(s'_{s_2,a}) = \Bar{s}'$.
\end{restatable}
\begin{proof}
    Since $s_1,s_2 \in \alpha^{-1}(\Bar{s})$, then $\alpha(s_1) = \alpha(s_2) = \Bar{s}$. Additionally, $s'_{s_1,a} = s_1 +f_a(e)$ and $s'_{s_2,a} = s_2 +f_a(e)$. This implies:
    \begin{align*}
        \alpha(s'_{s_1,a}) &= \alpha(s_1)+\alpha(f_a(e)) = \Bar{s} +\alpha(f_a(e)) \\ 
        &= \alpha(s_2) + \alpha(f_a(e)) = \alpha(s'_{s_2,a}) = \Bar{s}'.
    \end{align*}
\end{proof}
Proposition~\ref{prop:ambiguity_absence_pplanning} and the above observation about the transition function allow us to state the fact that:
\begin{align}
    \sum_{s'\in \alpha^{-1}(\Bar{s}')}\mathcal{T}(s'|s_1, a) &= \sum_{s'\in \alpha^{-1}(\Bar{s}')}\mathcal{T}(s'|s_2, a)\nonumber\\ 
    &\forall s_1, s_2 \in \alpha^{-1}(\Bar{s}), \label{eq:abstraction_collapse}
\end{align}

where $a$ is applicable. This implies that among all the states $s$ belonging to a certain abstract class $\alpha^{-1}(\Bar{s})$, where action $a$ is applicable, it does not matter what state we chose as representative of the abstract transition function for the abstract state $\Bar{s}$ and action $a$. Indeed through the lenses of the aggregation function they will all transition to the same abstract next states with the same probabilities. 
Even though Equation~\ref{eq:abstraction_collapse} was proved first by~\citealp{klossner2021pattern}, they do not explicitly link this fact to the absence of conditional effects and linearity of the aggregation function (used in Proposition~\ref{prop:ambiguity_absence_pplanning}). 

Under the above described properties, for what concerns ARMDPs, we will have the same abstract transition function for any choice of weights that satisfies the condition stated in Equation~\ref{eq:armdps_remapping_cond}.~\footnote{We do not consider weighting function abstractions because, as reported in Section~\ref{sec:wfa_mapping}, they cannot handle cases when some actions are not applicable to all the states in an abstract class.} Additionally, if there are states in the abstract class $\alpha^{-1}(\Bar{s})$ that, through $a$, are connected to a goal state, we have to require the more restrictive condition stated in Equation~\ref{eq:armdps_reward_cond} to preserve the reward coming from reaching the goal. These choices make the abstraction connection preserving and equivalent to the one in Probabilistic Planning. For what concerns, ABPMDPs, the above described properties coming from Probabilistic Planning without conditional effects imply that the interval associated to $\mathcal{\Bar{T}}_I(\Bar{s}'|\Bar{s}, a)$ is either one single point (in case $a$ is applicable in any $s\in \alpha^{-1}(\Bar{s})$) or its left boundary is $0$ (in case it exists a state $s\in\alpha^{-1}(\Bar{s})$ where $a$ is not applicable). Choosing the maximum for both transition and reward functions as we did in the deterministic case will yield the same abstraction as the one we have in Probabilistic Planning.

Finally, it is important to highlight the consequences implied by the absence of conditional effects from a broader perspective. While recompiling a problem to get rid of conditional effects is always possible in theory, in practice this may not be the case due to time/memory constraints. The compilation grows exponentially in size (and time) with the number of conditional effects within the same action. This representational issue goes up to the point that not even a (P)PDDL representation is feasible: an example would be Atari $2600$ Games~\cite{bellemare2013arcade, lipovetzky2015classical}. These games do not have many actions, but their effects heavily depend on the state they are applied in. On the one hand, tackling tasks that can be represented through an action schema free of conditional effects has many advantages. The most important within this paper's context being the possibility of computing the abstraction without needing the transition graph. On the other hand, this prevents us to tackle tasks with many conditional effects. In this last case, since Equation~\ref{eq:abstraction_collapse} does not hold anymore, via WFA, ARMDPs, and ABPMDPs we have access to a richer set of abstractions, even considering aggregation functions that are projections. New research is required to investigate ways that allow to automatically compute abstractions fast in these very complex practical scenarios that can be found in Probabilistic Planning as well as in Classical Planning.

\subsection{Abstractions for Discounted Markov Decision Processes in the Stochastic Shortest Path Setting}\label{sec:MDP_vs_SSP}
Abstractions for MDPs, namely WFAs, ARMDPs, and ABPMDPs, have been developed for discounted MDPs that are a subfamily of SSPs~\cite{natarajan2022planning}. Aggregating the transition function and the reward function as prescribed by the previously mentioned abstractions (under an arbitrary aggregation function) in the context of discounted MDPs will still yield a discounted MDP. This may not be the case if the same operations are applied in the context of SSPs.
An SSP is a tuple $\langle \mathcal{S}, \mathcal{A}, \mathcal{T}, \mathcal{R}, G\rangle$ where $\mathcal{S}, \mathcal{A},\text{ and }\mathcal{T}$ are like in a discounted MDP, and $\mathcal{R} \in \mathbb{R}\setminus\{0\}$ (to avoid the possibility of zero reward cycles, negative reward is usually associated to costs, and positive reward is usually given upon reaching the goal). Without loss of generality, there is a single goal state $s_G$ and a single action $a_G$ such that $\mathcal{R}(s_G, a_G) = 0$ and $\mathcal{T}(s_G|s_G, a_G) = 1$. Furthermore, there has to exist at least one proper policy, and for every improper stationary policy $\pi$, its value function must be $-\infty$ for at least a state where it is improper. A policy is said proper if it reaches the goal with probability $1$ (as the number of steps goes to infinity) from any state $s\in\mathcal{S}$. A policy that is not proper is said improper. The requirement that at least a proper policy has to exist has to do with the convergence of the expected returns. The fact that improper policies must have at least a state where their value function is infinitely negative is needed for solution algorithms to discard them.
The existence of a proper policy depends on the transition function of the MDP, hence, transforming the transition function of a SSP through abstraction mechanisms as they are designed for discounted MDPs, may not guarantee that at least a proper policy will be preserved in general. Therefore, we may go outside the SSP class of problems. Investigating under what conditions WFA, ARMDPs, and ABPMDPs are closed abstractions under the SSP framework is outside the scope of this paper and avenue for future work. The abstraction we discussed in the previous section falls within the family of SSPs (if we allow zero reward functions, then into the Generalized SSPs~\cite{klossner2021ssp}). This implies that there exist some ways of abstracting an SSP via ARMDPs and ABPMDPs that still result in an SSP.

\section{Conclusions}\label{sec:conclusions}
In this paper, we have unified projection abstractions in (Classical / Probabilistic) Planning and discounted MDPs showing how to obtain abstractions in discounted MDPs that are equivalent to the Planning ones. We have analyzed the role that the absence of conditional effects has onto abstractions in general. We have also highlighted the need for new efficient approaches able to (partially) remove this assumption at the benefit of both Planning and MDPs. In the context of the SSP framework, we emphasized the necessity for new theoretical studies aiming to find out under what conditions WFAs, ARMDPs, and ABPMDPs still generate abstractions that fall within the SSP family. All of this constitutes a first starting point to bring Planning and RL closer together under the concept of abstractions.

\section*{Disclaimer}
This paper was prepared for informational purposes  by the Artificial Intelligence Research group of JPMorgan Chase \& Co. and its affiliates ("JP Morgan'') and is not a product of the Research Department of JP Morgan. JP Morgan makes no representation and warranty whatsoever and disclaims all liability, for the completeness, accuracy or reliability of the information contained herein. This document is not intended as investment research or investment advice, or a recommendation, offer or solicitation for the purchase or sale of any security, financial instrument, financial product or service, or to be used in any way for evaluating the merits of participating in any transaction, and shall not constitute a solicitation under any jurisdiction or to any person, if such solicitation under such jurisdiction or to such person would be unlawful.

© 2024 JPMorgan Chase \& Co. All rights reserved.
\bibliography{aaai25}

\end{document}

%% file: macros.tex
\newcommand{\alberto}[1]{\textcolor{blue}{#1}}
\newcommand{\giuseppe}[1]{\textcolor{orange}{#1}}
\newcommand{\daniel}[1]{\textcolor{red}{#1}}

%% file: main-icaps.bbl
\begin{thebibliography}{23}
\providecommand{\natexlab}[1]{#1}

\bibitem[{Bai, Srivastava, and Russell(2016)}]{bai2016markovian}
Bai, A.; Srivastava, S.; and Russell, S. 2016.
\newblock Markovian State and Action Abstractions for MDPs via Hierarchical MCTS.
\newblock In \emph{IJCAI}, 3029--3039.

\bibitem[{Bellemare et~al.(2013)Bellemare, Naddaf, Veness, and Bowling}]{bellemare2013arcade}
Bellemare, M.~G.; Naddaf, Y.; Veness, J.; and Bowling, M. 2013.
\newblock The arcade learning environment: An evaluation platform for general agents.
\newblock \emph{Journal of Artificial Intelligence Research}, 47: 253--279.

\bibitem[{Bertsekas and Tsitsiklis(1991)}]{bertsekas1991analysis}
Bertsekas, D.~P.; and Tsitsiklis, J.~N. 1991.
\newblock An analysis of stochastic shortest path problems.
\newblock \emph{Mathematics of Operations Research}, 16(3): 580--595.

\bibitem[{Congeduti and Frans(2022)}]{congeduti2022cross}
Congeduti, E.; and Frans, A. 2022.
\newblock A Cross-Field Review of State Abstraction for Markov Decision Processes.
\newblock In \emph{34th Benelux Conference on Artificial Intelligence (BNAIC) and the 30th Belgian Dutch Conference on Machine Learning (Benelearn)}.

\bibitem[{Culberson and Schaeffer(1998)}]{culberson1998pattern}
Culberson, J.~C.; and Schaeffer, J. 1998.
\newblock Pattern databases.
\newblock \emph{Computational Intelligence}, 14(3): 318--334.

\bibitem[{Edelkamp(2002)}]{edelkamp2002symbolic}
Edelkamp, S. 2002.
\newblock Symbolic Pattern Databases in Heuristic Search Planning.
\newblock In \emph{AIPS}, 274--283.

\bibitem[{Fikes and Nilsson(1971)}]{fikes1971strips}
Fikes, R.~E.; and Nilsson, N.~J. 1971.
\newblock STRIPS: A new approach to the application of theorem proving to problem solving.
\newblock \emph{Artificial intelligence}, 2(3-4): 189--208.

\bibitem[{Ghallab et~al.(1998)Ghallab, Howe, Knoblock, McDermott, Ram, Veloso, Weld, and Wilkins}]{ghallab1998pddl}
Ghallab, M.; Howe, A.; Knoblock, C.; McDermott, D.; Ram, A.; Veloso, M.; Weld, D.; and Wilkins, D. 1998.
\newblock PDDL-the planning domain definition language.

\bibitem[{Givan, Leach, and Dean(2000)}]{givan2000bounded}
Givan, R.; Leach, S.; and Dean, T. 2000.
\newblock Bounded-parameter Markov decision processes.
\newblock \emph{Artificial Intelligence}, 122(1-2): 71--109.

\bibitem[{Haslum et~al.(2007)Haslum, Botea, Helmert, Bonet, Koenig et~al.}]{haslum2007domain}
Haslum, P.; Botea, A.; Helmert, M.; Bonet, B.; Koenig, S.; et~al. 2007.
\newblock Domain-independent construction of pattern database heuristics for cost-optimal planning.
\newblock In \emph{AAAI}, volume~7, 1007--1012.

\bibitem[{Helmert et~al.(2007)Helmert, Haslum, Hoffmann et~al.}]{helmert2007flexible}
Helmert, M.; Haslum, P.; Hoffmann, J.; et~al. 2007.
\newblock Flexible Abstraction Heuristics for Optimal Sequential Planning.
\newblock In \emph{ICAPS}, 176--183.

\bibitem[{Kl{\"o}{\ss}ner and Hoffmann(2021)}]{klossner2021ssp}
Kl{\"o}{\ss}ner, T.; and Hoffmann, J. 2021.
\newblock Pattern Databases for Stochastic Shortest Path Problems.
\newblock In \emph{Proceedings of the International Symposium on Combinatorial Search}, volume~12, 131--135.

\bibitem[{Kl{\"o}{\ss}ner et~al.(2021)Kl{\"o}{\ss}ner, Hoffmann, Steinmetz, and Torralba}]{klossner2021pattern}
Kl{\"o}{\ss}ner, T.; Hoffmann, J.; Steinmetz, M.; and Torralba, A. 2021.
\newblock Pattern databases for goal-probability maximization in probabilistic planning.
\newblock In \emph{Proceedings of the International Conference on Automated Planning and Scheduling}, volume~31, 201--209.

\bibitem[{Kolobov et~al.(2011)Kolobov, Mausam, Weld, and Geffner}]{kolobov2011heuristic}
Kolobov, A.; Mausam, M.; Weld, D.; and Geffner, H. 2011.
\newblock Heuristic search for generalized stochastic shortest path MDPs.
\newblock In \emph{Proceedings of the International Conference on Automated Planning and Scheduling}, volume~21, 130--137.

\bibitem[{Li, Walsh, and Littman(2006)}]{li2006towards}
Li, L.; Walsh, T.~J.; and Littman, M.~L. 2006.
\newblock Towards a unified theory of state abstraction for MDPs.
\newblock In \emph{AI\&M}.

\bibitem[{Lipovetzky, Ramirez, and Geffner(2015)}]{lipovetzky2015classical}
Lipovetzky, N.; Ramirez, M.; and Geffner, H. 2015.
\newblock Classical planning with simulators: Results on the atari video games.
\newblock In \emph{Proc. IJCAI}.

\bibitem[{Natarajan and Kolobov(2022)}]{natarajan2022planning}
Natarajan, M.; and Kolobov, A. 2022.
\newblock \emph{Planning with Markov decision processes: An AI perspective}.
\newblock Springer Nature.

\bibitem[{Petrik and Subramanian(2014)}]{petrik2014raam}
Petrik, M.; and Subramanian, D. 2014.
\newblock RAAM: The benefits of robustness in approximating aggregated MDPs in reinforcement learning.
\newblock \emph{Advances in Neural Information Processing Systems}, 27.

\bibitem[{Puterman(2014)}]{puterman2014markov}
Puterman, M.~L. 2014.
\newblock \emph{Markov decision processes: discrete stochastic dynamic programming}.
\newblock John Wiley \& Sons.

\bibitem[{Seipp and Helmert(2018)}]{seipp2018counterexample}
Seipp, J.; and Helmert, M. 2018.
\newblock Counterexample-guided Cartesian abstraction refinement for classical planning.
\newblock \emph{Journal of Artificial Intelligence Research}, 62: 535--577.

\bibitem[{Sutton and Barto(2018)}]{sutton2018reinforcement}
Sutton, R.~S.; and Barto, A.~G. 2018.
\newblock \emph{Reinforcement Learning: An Introduction}.
\newblock Cambridge, MA, USA: A Bradford Book.
\newblock ISBN 0262039249.

\bibitem[{Torralba and Croitoru(2018)}]{torralba2018aipl}
Torralba, A.; and Croitoru, C. 2018.
\newblock AI Planning: Pattern Database Heuristics.
\newblock Accessed on October 31, 2024.

\bibitem[{Younes and Littman(2004)}]{younes2004ppddl1}
Younes, H.~L.; and Littman, M.~L. 2004.
\newblock PPDDL1. 0: An extension to PDDL for expressing planning domains with probabilistic effects.
\newblock \emph{Techn. Rep. CMU-CS-04-162}, 2: 99.

\end{thebibliography}
